\documentclass{article}

\usepackage[preprint]{neurips_2019}

\usepackage[utf8]{inputenc} 
\usepackage[T1]{fontenc}    
\usepackage{hyperref}       
\usepackage{url}            
\usepackage{booktabs}       
\usepackage{amsfonts}       
\usepackage{nicefrac}       
\usepackage{microtype}      
\usepackage{amsmath}
\usepackage{pifont}
\usepackage{amsthm}
\usepackage{algorithm,algorithmicx,algpseudocode}
\usepackage{mathtools}
\usepackage{xcolor}
\usepackage{comment}
\usepackage{bbm}

\usepackage{caption}
\usepackage{subcaption}

\newtheorem{thm}{Theorem}
\newtheorem{lem}{Lemma}
\newtheorem{claim}{Claim}
\newtheorem{cor}{Corollary}
\newtheorem{rmk}{Remark}

\newtheorem{prop}{Proposition}

\newcommand{\eg}{\textit{e.g.,} }

\newcommand{\R}{\mathbb{R}}

\newcommand{\thmref}[1]{Theorem~\ref{#1}}

\renewcommand{\eqref}[1]{(\ref{#1})}

\newcommand{\currw}{x_t}
\newcommand{\nextw}{x_{t+1}}

\newcommand{\wi}{x_i}
\newcommand{\currwi}{x_{t,i}}
\newcommand{\w}{x}
\newcommand{\feas}{\mathcal F}
\newcommand{\proj}{\Pi_\feas}
\newcommand{\fs}{f_t}

\newcommand{\currg}{g_t}

\newcommand{\prevm}{m_{t-1}}
\newcommand{\currm}{m_t}
\newcommand{\prevv}{v_{t-1}}
\newcommand{\currv}{v_t}

\newcommand{\curre}{\eta_t}
\newcommand{\currei}{\eta_{t,i}}
\newcommand{\prevei}{\eta_{t-1,i}}

\newcommand{\curra}{\alpha_t}

\newcommand{\normed}[1]{\left\lVert {#1} \right\rVert}

\newcommand{\btwo}{\beta_2}
\newcommand{\bone}{\beta_1}

\newcommand{\bonet}{\beta_{1t}}

\newcommand{\expec}[2]{\mathbb E_{#1} \left[ {#2} \right]}

\title{On the Convergence of AdaBound \\ and its Connection to SGD}

\author{%
  Pedro Savarese \\
  Toyota Technological Institute at Chicago\\
  \texttt{savarese@ttic.edu} \\
}

\begin{document}

\maketitle

\begin{abstract}
   Adaptive gradient methods such as Adam have gained extreme popularity due to their success in training complex neural networks and less sensitivity to hyperparameter tuning compared to SGD. However, it has been recently shown that Adam can fail to converge and might cause poor generalization -- this lead to the design of new, sophisticated adaptive methods which attempt to generalize well while being theoretically reliable.
   
   In this technical report we focus on AdaBound, a promising, recently proposed optimizer. We present a stochastic convex problem for which AdaBound can provably take arbitrarily long to converge in terms of a factor which is \textit{not} accounted for in the convergence rate guarantee of \cite{adabound}. We present a new $O(\sqrt T)$ regret guarantee under different assumptions on the bound functions, and provide empirical results on CIFAR suggesting that a specific form of momentum SGD can match AdaBound's performance while having less hyperparameters and lower computational costs.
\end{abstract}

\section{Introduction}
\label{sec-intro}

We consider first-order optimization methods which are concerned with problems of the following form:
\begin{equation}
    \min_{\w \in \feas} f(\w)
\label{eq:min}
\end{equation}
where $\feas \subseteq \R^d$ is the feasible set of solutions and $f : \R^d \to \R$ is the objective function. First-order methods typically operate in an iterative fashion: at each step $t$, the current candidate solution $\currw$ is updated using both zero-th and first-order information about $f$ (\eg $f(\currw)$ and $\nabla f(\currw)$, or unbiased estimates of each). Methods such as gradient descent and its stochastic counterpart can be written as:
\begin{equation}
    \nextw = \proj \left( \currw - \curra \cdot \currm \right)
\end{equation}

where $\curra \in \R$ is the learning rate at step $t$, $\currm \in \R^d$ is the update direction (\eg $\nabla f(\currw)$ for deterministic gradient descent), and $\Pi_\feas$ denotes a projection onto $\feas$. The behavior of vanilla gradient-based methods is well-understood under different frameworks and assumptions ($O(\sqrt T)$ regret in the online convex framework \citep{onlinegd}, $O(1 / \sqrt T)$ suboptimality in the stochastic convex framework, and so on).

In contrast with SGD, adaptive gradient methods such as AdaGrad \citep{adagrad}, RMSProp \citep{rmsprop} and Adam \citep{adam} propose to compute a different learning rate for each parameter in the model. In particular, the parameters are updated according to the following rule:
\begin{equation}
    \nextw = \proj \left( \currw - \curre \odot \currm \right)
\end{equation}
where $\curre \in \R^d$ are parameter-wise learning rates and $\odot$ denotes element-wise multiplication. For Adam, we have $\currm = \bone \prevm + (1 - \bone) \currg$ and $\curre = 1 / \sqrt \currv$ with $\currv = \btwo \prevv + (1 - \btwo) \currg^2$, where $\currg$ captures first-order information of the objective function (\eg $\expec{}{\currg} = \nabla f(\currw)$ in the stochastic setting).

Adaptive methods have become popular due to their flexibility in terms of hyperparameters, which require less tuning than SGD. In particular, Adam is currently the de-facto optimizer for training complex models such as BERT \citep{bert} and VQ-VAE \citep{vqvae}.

Recently, it has been observed that Adam has both theoretical and empirical gaps. \cite{amsgrad} showed that Adam can fail to converge in the stochastic convex setting, while \cite{marginal} have formally demonstrated that Adam can cause poor generalization -- a fact often observed when training CNN-like models such as ResNets \citep{resnet1}. These shortcomings have motivated the design of new adaptive methods, with the end-goal of achieving strong theoretical guarantees and SGD-like empirical performance: most notably, \cite{amsgrad} propose AMSGrad, which was shown to have the same regret rate as SGD for online convex optimization. \cite{avagrad} propose Delayed Adam and AvaGrad -- both which enjoy the same convergence rate as SGD in the stochastic non-convex setting -- and show that, with careful tuning, adaptive methods can match or even best SGD's empirical performance. Nonetheless, there has been continued effort in further improving adaptive methods.

AdaBound \citep{adabound} is a recently proposed adaptive gradient method that aims to bridge the empirical gap between Adam-like methods and SGD, and consists of enforcing dynamic bounds on $\curre$ such that as $t$ goes to infinity, $\curre$ converges to a vector whose components are equal -- hence degenerating to SGD. AdaBound comes with a $O(\sqrt T)$ regret rate in the online convex setting, yielding an immediate guarantee in the stochastic convex framework due to \cite{online2batch}. Moreover, empirical experiments suggest that it is capable of outperforming SGD in image classification tasks -- problems where adaptive methods have historically failed to provide competitive results.

In Section \ref{sec-slow}, we highlight issues in the convergence rate proof of AdaBound (Theorem 4 of \cite{adabound}), and present a stochastic convex problem for which AdaBound can take \textit{arbitrarily} long to converge. More importantly, we show that the presented problem leads to a contradiction with the convergence guarantee of AdaBound \textit{while satisfying all of its assumptions}, implying that Theorem 4 of \cite{adabound} is indeed incorrect. In Section \ref{sec-convergence}, we introduce a new assumption which yields a $O(\sqrt T)$ regret guarantee without assuming that the bound functions are monotonic nor that they converge to the same limit. Driven by the new guarantee, in Section \ref{sec-experiments} we re-evaluate the performance of AdaBound on the CIFAR dataset, and observe that its performance can be matched with a specific form of SGDM \footnote{Available at \url{https://github.com/lolemacs/adabound-and-csgd}}, whose computational cost is significantly smaller than that of Adam-like methods.

\section{Notation}
\label{sec-notation}

For vectors $a = [a_1, a_2, \dots], b = [b_1, b_2, \dots] \in \R^d$ and scalar $b \in \R$, we use the following notation: $\frac1{a}$ for element-wise division ($\frac1{a} = [\frac1{a_1}, \frac1{a_2}, \dots]$), $\sqrt a$ for element-wise square root ($\sqrt a = [\sqrt{a_1}, \sqrt{a_2}, \dots]$), $a + b$ for element-wise addition ($a+b = [a_1+b, a_2+b, \dots]$), $a \odot b$ for element-wise multiplication ($a \odot b = [a_1 b_1, a_2 b_2, \dots]$). Moreover, $\normed{a}$ is used to denote the $\ell_2$-norm: other norms will be specified whenever used (\eg $\normed{a}_\infty$).

For subscripts and vector indexing, we adopt the following convention: the subscript $t$ is used to denote an object related to the $t$-th iteration of an algorithm (\eg $\currw \in \R^d$ denotes the iterate at time step $t$); the subscript $i$ is used for indexing: $\wi \in \R$ denotes the $i$-th coordinate of $\w \in \R^d$. When used together, $t$ precedes $i$: $\currwi \in \R$ denotes the $i$-th coordinate of $\currw \in \R^d$.
\section{AdaBound's Arbitrarily Slow Convergence}
\label{sec-slow}

\begin{algorithm}
    \caption{\textsc{AdaBound}}
    \textbf{Input:} $\w_1 \in \R^d$, initial step size $\alpha$, $\{\bonet\}_{t=1}^T$, $\beta_2$, lower bound function $\eta_l$, upper bound function $\eta_u$
    \begin{algorithmic}[1]
    \State Set $m_0 = 0, v_0 = 0$
    \For{$t = 1$ \textbf{to} $T$}
    \State $\currg = \nabla \fs(\currw)$
    \State $\currm = \bonet \prevm + (1-\bonet) \currg$
    \State $\currv = \btwo \prevv + (1-\btwo) \currg^2$
    \State $\hat{\curre} = \mathrm{Clip}(\alpha/\sqrt \currv, \eta_l(t), \eta_u(t))$ and $\curre = \hat{\eta}_t / \sqrt{t}$
    \State $\nextw = \Pi_{\feas, \mathrm{diag}(\eta_t^{-1})} (\currw - \eta_t \odot m_t)$
    \EndFor
    \end{algorithmic}
\label{alg:adabound}
\end{algorithm}

AdaBound is given as Algorithm \ref{alg:adabound}, following \citep{adabound}. It consists of an update rule similar to Adam, except for the extra element-wise clipping operation $\mathrm{Clip}(\alpha/\sqrt \currv, \eta_l(t), \eta_u(t)) = \max \left( \min(\alpha / \sqrt \currv, \eta_u(t)), \eta_l(t) \right)$, which assures that $\hat{\currei} \in [\eta_l(t), \eta_u(t)]$ for all $i \in [d]$. The bound functions $\eta_l, \eta_u$ are chosen such that $\eta_l(t)$ is non-decreasing, $\eta_u(t)$ is non-increasing, and $\lim_{t \to \infty} \eta_l(t) = \lim_{t \to \infty} \eta_u(t) = \alpha^\star$, for some $\alpha^\star$. It then follows that $\lim_{t \to \infty} \hat{\curre} = \vec 1 \alpha^\star$, thus AdaBound degenerates to SGD in the time limit.

In \citep{adabound}, the authors present the following Theorem:

\begin{thm}(Theorem 4 of \cite{adabound})
\label{thm:adaboundorig}
Let $\{x_t\}$ and $\{v_t\}$ be the sequences obtained from Algorithm~\ref{alg:adabound}, $\beta_1 = \beta_{11}$, $\beta_{1t} \leq \beta_1$ for all $t \in [T]$ and $\beta_1 / \sqrt{\beta_2} < 1$.
Suppose $\eta_l(t+1) \geq \eta_l(t) > 0$, $\eta_u(t+1) \leq \eta_u(t)$, $\eta_l(t) \rightarrow \alpha^*$ as $t \rightarrow \infty$, $\eta_u(t) \rightarrow \alpha^*$ as $t \rightarrow \infty$, $L_\infty = \eta_l(1)$ and $R_\infty = \eta_u(1)$.
Assume that $\|x - y\|_\infty \leq D_\infty$ for all $x, y \in \mathcal{F}$ and $\|\nabla f_t(x)\| \leq G_2$ for all $t \in [T]$ and $x \in \mathcal{F}$. For $x_t$ generated using the \textsc{AdaBound} algorithm, we have the following bound on the regret
\begin{equation}
    R_T \leq \frac{D_\infty^2 \sqrt{T}}{2(1-\beta_1)} \sum_{i=1}^d \hat{\eta}_{T,i}^{-1} + \frac{D_\infty^2}{2(1-\beta_1)} \sum_{t=1}^T \sum_{i=1}^d \beta_{1t} \eta_{t,i}^{-1} + (2\sqrt{T} - 1) \frac{R_\infty G_2^2}{1-\beta_{1}}
\end{equation}
\end{thm}

Its proof claims that $\currei^{-1} \geq \prevei^{-1}$ follows from the definition of $\curre$ in AdaBound, a fact that only generally holds if $\frac{\eta_l(t-1)}{\eta_u(t)} \geq \sqrt{1 - \frac1t}$ for all $t$. Even for the bound functions $\eta_l(t) = 1 - \frac{1}{\gamma t + 1}, \eta_u(t) = 1 + \frac{1}{\gamma t}$ considered in \cite{adabound} and used in the released code, this requirement is not satisfied for any $\gamma > 0$. Finally, it is also possible to show that AMSBound does not meet this requirement either, hence the proof of Theorem 5 of \cite{adabound} is also problematic.

It turns out that the convergence of AdaBound in the stochastic convex case can be arbitrarily slow, even for bound functions that satisfy the assumptions in \thmref{thm:adaboundorig}:

\begin{thm}
\label{the:sto-example}
For any constant $\beta_1, \beta_2 \in [0,1)$, $\beta_1 < \sqrt{\beta_2}$ and initial step size $\alpha$, there exist bound functions $\eta_l(t;K), \eta_u(t;K)$ such that $\eta_l(t+1; K) \geq \eta_l(t; K) > 0$, $\eta_u(t+1; K) \leq \eta_u(t; K)$, $\lim_{t \to \infty} \eta_l(t; K) = \lim_{t \to \infty} \eta_u(t; K)$, and a stochastic convex optimization problem for which the iterates produced by \textsc{AdaBound} satisfy $\expec{}{f(\currw)} - f(\w^*) \geq 1$ for all $t \leq K$.
\end{thm}

\begin{proof}
    We consider the same stochastic problem as presented in \cite{amsgrad}, for which Adam fails to converge. In particular, a one-dimensional problem over $[-1,1]$, where $\fs$ is chosen i.i.d. as follows:
    \begin{equation}
        \fs(\w)  =
            \begin{cases}
            C \w, \quad \text{with probability } \quad p \coloneqq \frac{1+\delta}{C+1} \\
            -\w, \quad  \text{with probability } \quad 1-p
            \end{cases}
    \end{equation}
    Here, $C$ is taken to be large in terms of $\delta, \bone$ and $\btwo$, and $C > \max(1, \alpha)$. Now, consider the following bound functions:
    \begin{equation}
        \eta_l(t) = \alpha / C
        \quad\quad\quad
        \eta_u(t;K) =
        \begin{cases}
        \alpha / \sqrt{1 - \btwo}, \quad \text{for $t \leq K$}\\
        \alpha / C, \quad \text{otherwise}
        \end{cases}
    \end{equation}
    and check that $\eta_l$ and $\eta_u$ are non-decreasing and non-increasing in $t$, respectively, and $\lim_{t \to \infty} \eta_l(t) = \lim_{t \to \infty} \eta_u(t;K) = \alpha / C$. We will show that such bound functions can be effectively ignored for $t \leq K$. Check that, for all $t$:
    \begin{equation}
    \begin{split}
        & \currv = (1 - \btwo) \sum_{i=1}^t \btwo^{t-i} g_i^2 \leq C^2 (1 - \btwo^{t}) \leq C^2 \\
        & \currv = (1 - \btwo) \sum_{i=1}^t \btwo^{t-i} g_i^2 \geq 1 - \btwo^{t} \geq 1 - \btwo
    \end{split}
    \end{equation}
    where we used the fact that $\currg^2 \in \{1, C^2\}$ and that $C > 1$. Hence, we have, for $t \leq K$:
    \begin{equation}
        \begin{split}
            \eta_l(t) = \frac{\alpha}{C} \leq \frac{\alpha}{\sqrt \currv} \leq \frac{\alpha}{\sqrt{1-\btwo}} = \eta_u(t;K)
        \end{split}
    \end{equation}
    Since $\eta_l(t) \leq \frac{\alpha}{\sqrt \currv} \leq \eta_u(t;K)$ for all $t \leq K$, the clipping operation acts as an identity mapping and $\hat \curre = \frac{\alpha}{\sqrt \currv}$. Therefore, in this setting, AdaBound produces the same iterates $\{\currw\}_{t=1}^K$ as Adam. We can then invoke Theorem 3 of \cite{amsgrad}, and have that, with $C$ large enough (as a function of $\delta, \bone, \btwo$), for all $t \leq K$, we have $\expec{}{\nextw} \geq \expec{}{\currw}$. In particular, with $\w_1 \geq 0$, $\expec{}{\currw} \geq 0$  for all $t \leq K$ and hence $\expec{}{f(\currw)} - f(\w^*) \geq - f(\w^*) = \delta$. Setting $\delta = 1$ finishes the proof.
\end{proof}

While the bound functions considered in the Theorem above might seem artificial, the same result holds for bound functions of the form $\eta_l(t) = 1 - \frac{1}{\gamma t + 1}$ and $\eta_u(t) = 1 + \frac{1}{\gamma t}$, considered in \cite{adabound} and in the publicly released implementation of AdaBound:

\begin{claim}
    \thmref{the:sto-example} also holds for the bound functions $\eta_l(t; K) = 1 - \frac{1}{\gamma t + 1}$ and $\eta_u(t; K) = 1 + \frac{1}{\gamma t}$ with
    \begin{equation}
        \gamma = \frac1{K} \cdot \min \left( \frac{\alpha}{C-\alpha}, \frac{\sqrt{1-\btwo}}{\alpha} \right)
    \end{equation}
\end{claim}
\begin{proof}
    Check that, for all $t \leq K$:
    \begin{equation}
    \begin{split}
        & \eta_l(t;K) \leq 1 - \frac{1}{\gamma K + 1} \leq 1 - \frac{1}{\frac{\alpha}{C-\alpha} + 1} =  1 - \frac{C-\alpha}{C} = \frac{\alpha}{C} \\
        & \eta_u(t;K) \geq 1 + \frac{1}{\gamma K} \geq 1 + \frac{\alpha}{\sqrt{1 - \btwo}} \geq \frac{\alpha}{\sqrt{1 - \btwo}}
    \end{split}
    \end{equation}
    Hence, for the stochastic problem in \thmref{the:sto-example}, we also have that $\eta_l(t;K) \leq \frac{\alpha}{\sqrt \currv} \leq \eta_u(t;K)$ for all $t \leq K$.
\end{proof}

Note that it is straightforward to prove a similar result for the online convex setting by invoking Theorem 2 instead of Theorem 3 of \cite{amsgrad} -- this would immediately imply that \thmref{thm:adaboundorig} is incorrect. Instead, \thmref{the:sto-example} was presented in the convex stochastic setup as it yields a stronger result, and it almost immediately implies that \thmref{thm:adaboundorig} might not hold:

\begin{cor}
\label{cor:incorrect}
There exists an instance where \thmref{thm:adaboundorig} does not hold.
\end{cor}

\begin{proof}
    Consider AdaBound with the bound functions presented in \thmref{the:sto-example} and $\bone = 0$. For any sequence $\{\fs\}_{t=1}^K$ drawn for the stochastic problem in \thmref{the:sto-example}, setting $D_\infty = 2, G_2 = C$ and $R_\infty = \alpha / \sqrt{1 - \btwo}$ in \thmref{thm:adaboundorig} yields, for $T = K$:
    \begin{equation}
    \label{eq:wrongregret}
        R_K = \sum_{t=1}^K \left(\fs(\currw) - \fs(\w^*)\right) \leq \frac{2 d C \sqrt{K}}{\alpha} + (2\sqrt{K} - 1) \frac{\alpha C^2}{\sqrt{1 - \btwo}}
    \end{equation}
    where we used the fact that $\hat \eta_{K,i}^{-1} = \frac{\sqrt{ v_{K,i}} }{\alpha} \leq \frac{C}{\alpha}$. Pick $K$ large enough such that $R_K < \frac{K}{100}$. Taking expectation over sequences and dividing by $K$:
    \begin{equation}
    \label{eq:wrongregret}
        \frac{1}{K} \sum_{t=1}^K \expec{}{f(\currw)} - f(\w^*) < 0.01
    \end{equation}
    However, \thmref{the:sto-example} assures $\expec{}{f(\currw)} - f(\w^*) \geq 1$ for all $t \leq K$, raising a contradiction.
\end{proof}

Note that while the above result shows that \thmref{thm:adaboundorig} is indeed incorrect, it does \textit{not} imply that AdaBound might fail to converge.

\section{A New Guarantee}
\label{sec-convergence}

The results in the previous section suggest that \thmref{thm:adaboundorig} fails to capture all relevant properties of the bound functions. Although it is indeed possible to show that $R_T / T \to 0$, it is not clear whether a $O(\sqrt T)$ regret rate can be guaranteed for general bound functions.

It turns out that replacing the previous requirements on the bound functions by the assumption that $\frac{t}{\eta_l(t)} - \frac{t-1}{\eta_u(t-1)} \leq M$ for all $t$ suffices to guarantee a regret of $O(\sqrt T)$:

\begin{thm}
\label{the:adabound-corrected}
Let $\{x_t\}$ and $\{v_t\}$ be the sequences obtained from Algorithm~\ref{alg:adabound}, $\beta_1 = \beta_{11}$, $\beta_{1t} \leq \beta_1$ for all $t \in [T]$ and $\beta_1 / \sqrt{\beta_2} < 1$. Suppose $\eta_u(t) \leq R_\infty$ and $\frac{t}{\eta_l(t)} - \frac{t-1}{\eta_u(t-1)} \leq M$ for all $t \in [T]$.
Assume that $\|x - y\|_\infty \leq D_\infty$ for all $x, y \in \mathcal{F}$ and $\|\nabla f_t(x)\| \leq G_2$ for all $t \in [T]$ and $x \in \mathcal{F}$. For $x_t$ generated using the \textsc{AdaBound} algorithm, we have the following bound on the regret
\begin{equation}
    R_T \leq \frac{D_\infty^2}{2(1-\bone)} \Bigg[ 2d M (\sqrt T - 1) + \sum_{i=1}^d \bigg[ \eta_{1,i}^{-1} + \sum_{t = 1}^{T} \bonet \eta_{t,i}^{-1} \bigg] \Bigg] + (2\sqrt{T} - 1) \frac{R_\infty G_2^2}{1-\beta_{1}}
\end{equation}
\end{thm}

\begin{proof}
We start from an intermediate result of the original proof of Theorem 4 of \cite{adabound}:

\begin{lem}
\label{lem:adabound-intermediate}
    For the setting in Theorem \ref{the:adabound-corrected}, we have:
\begin{equation}
\label{eq:lemma-bound}
\begin{split}
    R_T &\leq \underbrace{\sum_{t = 1}^{T} \frac{1}{2(1-\beta_{1t})}\bigg[ \|\eta_t^{-1/2} \odot (x_t - x^*)\|^2 - \|\eta_t^{-1/2} \odot (x_{t+1} - x^*)\|^2 \bigg]}_{S_1} \\& \quad\quad + \underbrace{\sum_{t = 1}^{T} \frac{\beta_{1t}}{2(1-\beta_{1t})} \|\eta_t^{-1/2} \odot (x_t - x^*)\|^2}_{S_2} + (2\sqrt{T} - 1) \frac{R_\infty G_2^2}{1-\beta_{1}}
\end{split}
\end{equation}
\end{lem}

\begin{proof}
    The result follows from the proof of Theorem 4 in \cite{adabound}, up to (but not including) Equation 6.
\end{proof}

We will proceed to bound $S_1$ and $S_2$ from the above Lemma. Starting with $S_1$:
\begin{equation}
\label{eq:s1-bound}
\begin{split}
    S_1 & = \sum_{i=1}^d \sum_{t = 1}^{T} \frac{1}{2(1-\bonet)} \Bigg[ \eta_{t,i}^{-1} (x_{t,i} - x_i^*)^2 - \eta_{t,i}^{-1} (x_{t+1,i} - x_i^*)^2  \Bigg] \\
    & \leq \sum_{i=1}^d \Bigg[ \frac{\eta_{1,i}^{-1} (x_{1,i} - x_i^*)^2}{2(1-\beta_{11})} + \sum_{t = 2}^{T}  \left[ \frac{\eta_{t,i}^{-1}}{2(1-\beta_{1t})} - \frac{\eta_{t-1,i}^{-1}}{2(1-\beta_{1(t-1)})} \right] (x_{t,i} - x_i^*)^2   \Bigg] \\
    & \leq \sum_{i=1}^d \Bigg[ \frac{\eta_{1,i}^{-1} (x_{1,i} - x_i^*)^2}{2(1-\beta_{11})} + \sum_{t = 2}^{T}  \frac{\left[ \eta_{t,i}^{-1} - \eta_{t-1,i}^{-1} \right]}{2(1-\beta_{1(t-1)})} (x_{t,i} - x_i^*)^2 \Bigg] \\
    & \leq \sum_{i=1}^d \Bigg[ \frac{\eta_{1,i}^{-1} (x_{1,i} - x_i^*)^2}{2(1-\beta_{11})} + \sum_{t = 2}^{T}  \left[ \frac{\sqrt{t}}{\eta_l(t)} - \frac{\sqrt{t-1}}{\eta_u(t-1)} \right] \frac{(x_{t,i} - x_i^*)^2}{2(1-\beta_{1(t-1)})} \Bigg] \\
    & \leq \sum_{i=1}^d \Bigg[ \frac{\eta_{1,i}^{-1} (x_{1,i} - x_i^*)^2}{2(1-\beta_{11})} + \sum_{t = 2}^{T}  \frac{1}{\sqrt t} \left[ \frac{t}{\eta_l(t)} - \frac{t-1}{\eta_u(t-1)} \right] \frac{(x_{t,i} - x_i^*)^2}{2(1-\beta_{1(t-1)})} \Bigg] \\
    & \leq \sum_{i=1}^d \Bigg[ \frac{\eta_{1,i}^{-1} (x_{1,i} - x_i^*)^2}{2(1-\beta_{11})} + \sum_{t = 2}^{T} \frac{M}{\sqrt t} \cdot  \frac{(x_{t,i} - x_i^*)^2}{2(1-\beta_{1(t-1)})} \Bigg] \\
    & \leq  \frac{D_\infty^2}{2(1-\bone)} \sum_{i=1}^d \Bigg[ \eta_{1,i}^{-1} + 2 M (\sqrt T - 1) \Bigg] \\
    & = \frac{D_\infty^2}{2(1-\bone)} \Bigg[ 2dM (\sqrt T - 1) + \sum_{i=1}^d \eta_{1,i}^{-1} \Bigg] \\
\end{split}
\end{equation}

In the second inequality we used $0 \leq \beta_{1,(t+1)} \leq \beta_{1,t} < 1$, in the third the definition of $\eta_t$ along with the fact that $\eta_{t,i} \in [\eta_l(t), \eta_u(t)]$ for all $i \in [d]$ and $t \in [T]$, in the fifth the assumption that $\frac{t}{\eta_l(t)} - \frac{t-1}{\eta_u(t-1)} \leq M$, and in the sixth we used the bound $D_\infty$ on the feasible region, along with $\sum_{t=2}^T t^{-1/2} \leq 2 \sqrt T - 2$ and $\bonet \leq \bone$ for all $t$.

For $S_2$, we have:
\begin{equation}
\label{eq:s2-bound}
\begin{split}
    S_2 & = \sum_{i=1}^d \sum_{t = 1}^{T} \frac{1}{2(1-\bonet)} \bonet \eta_{t,i}^{-1} (x_{t,i} - x_i^*)^2  \leq \frac{D_\infty^2}{2(1-\bone)} \sum_{i=1}^d \sum_{t = 1}^{T}  \bonet \eta_{t,i}^{-1} \\
\end{split}
\end{equation}
where we used the $D_\infty$ bound on the feasible region, and the fact that $\bonet \leq \bone$ for all $t$.

Combining \eqref{eq:s1-bound} and \eqref{eq:s2-bound} into \eqref{eq:lemma-bound}, we get:
\begin{equation}
\begin{split}
    R_T & \leq \frac{D_\infty^2}{2(1-\bone)} \Bigg[ 2d M (\sqrt T - 1) + \sum_{i=1}^d \bigg[ \eta_{1,i}^{-1} + \sum_{t = 1}^{T} \bonet \eta_{t,i}^{-1} \bigg] \Bigg] + (2\sqrt{T} - 1) \frac{R_\infty G_2^2}{1-\beta_{1}}
\end{split}
\end{equation}
\end{proof}

The above regret guarantee is similar to the one in Theorem 4 of \cite{adabound}, except for the term $M$ which accounts for assumption introduced. Note that \thmref{the:adabound-corrected} does \textit{not} require $\eta_u(t)$ to be non-increasing, $\eta_l(t)$ to be non-decreasing, nor that $\lim_{t \to \infty} \eta_u(t) = \lim_{t \to \infty} \eta_l(t)$.

It is easy to see that the assumption indeed holds for the bound functions in \cite{adabound}:

\begin{prop}
\label{prop:bound}
    For the bound functions
    \begin{equation}
        \eta_l(t) = 1 - \frac{1}{\gamma t + 1}
        \quad\quad\quad
        \eta_u(t) = 1 + \frac{1}{\gamma t}
    \end{equation}
    if $\gamma > 0$, we have:
    \begin{equation}
        \frac{t}{\eta_l(t)} - \frac{t-1}{\eta_u(t-1)} \leq 3 + 2 \gamma^{-1}
    \end{equation}
\end{prop}

\begin{proof}

First, check that $\eta_l(t) = \frac{\gamma t}{\gamma t + 1}$ and $\eta_u(t) = \frac{\gamma t + 1}{\gamma t}$. Then, we have:
\begin{equation}
\begin{split}
    \frac{t}{\eta_l(t)} - \frac{t-1}{\eta_u(t-1)} & = t \left( \frac{1}{\eta_l(t)} - \frac{1}{\eta_u(t-1)} \right) + \frac{1}{\eta_u(t-1)} \\
    & \leq t \left( \frac{1}{\eta_l(t)} - \frac{1}{\eta_u(t-1)} \right) + 1 \\
    & = t \left( \frac{\gamma t + 1}{\gamma t} - \frac{\gamma(t-1)}{\gamma(t-1) + 1} \right) + 1 \\
    & = \frac{1}{\gamma} \left( \frac{2 \gamma t + 1 - \gamma}{\gamma(t-1) + 1} \right) + 1 \\
    & \leq \frac{2}{\gamma} \left( \frac{\gamma t + 1}{\gamma(t-1) + 1} \right) + 1 \\
    & \leq \frac{2}{\gamma} \left(1 + \gamma \right) + 1 \\
    & = 3 + 2 \gamma^{-1} \\
\end{split}
\end{equation}
In the first inequality we used $\eta_u(t)^{-1} \leq 1$ for all $t$, and in the last the fact that $\frac{\gamma t + 1}{\gamma(t-1) + 1} \leq 1 + \gamma$ for all $\gamma$ and $t \geq 1$, which is equivalent to $\gamma^2 (t-1) \geq 0$.
\end{proof}

With this in hand, we have the following regret bound for AdaBound:

\begin{cor}
\label{cor:AdaBound}
Suppose $\beta_{1t} = \beta_1/t$, $\eta_l(t) = 1 - \frac{1}{\gamma t+1}$, and $\eta_u(t) = 1 + \frac{1}{\gamma t}$ for $\gamma > 0$ in Theorem~\ref{the:adabound-corrected}. Then, we have:
\begin{equation}
    R_T \leq \frac{5 \sqrt T}{1 - \bone} \left(1 +  \gamma^{-1}\right) \left(d D_\infty^2 + G_2^2\right)
\end{equation}
\end{cor}

\begin{proof}

From the bound in Theorem \ref{the:adabound-corrected}, it follows that:
\begin{equation}
\label{eq:final-bound}
\begin{split}
    R_T & \leq \frac{D_\infty^2}{2(1-\bone)} \Bigg[ 2d M (\sqrt T - 1) + (1+\gamma^{-1}) \sum_{i=1}^d \bigg[ 1 + \bone \sum_{t = 1}^{T} \frac{\sqrt{t}}{t} \bigg] \Bigg] + (2\sqrt{T} - 1) \frac{(1 + \gamma^{-1}) G_2^2}{1-\beta_{1}} \\
    & \leq \frac{D_\infty^2}{2(1-\bone)} \Bigg[ 2d M(\sqrt T - 1) + d (1+\gamma^{-1}) \bigg[ 1 + \bone (2 \sqrt T - 1) \bigg] \Bigg] + (2\sqrt{T} - 1) \frac{(1 + \gamma^{-1}) G_2^2}{1-\beta_{1}} \\
    & \leq \frac{d D_\infty^2}{2(1-\bone)} \Bigg[2 \sqrt T \left(M + 1 + \gamma^{-1} \right) + 1 + \gamma^{-1} \Bigg] + 2\sqrt{T} \frac{(1 + \gamma^{-1}) G_2^2}{1-\beta_{1}} \\
    & \leq \frac{d D_\infty^2}{2(1-\bone)} \Bigg[2 \sqrt T \left(4 + 3 \gamma^{-1}\right) + 1 + \gamma^{-1} \Bigg] + 2\sqrt{T} \frac{(1 + \gamma^{-1}) G_2^2}{1-\beta_{1}} \\
    & \leq \frac{5 \sqrt T}{1 - \bone} \left(1 +  \gamma^{-1}\right) \left(d D_\infty^2 + G_2^2\right)
\end{split}
\end{equation}
In the first inequality we used the facts that $\currei^{-1} = \sqrt t \hat \eta_{t,i}^{-1} \leq \sqrt t \eta_l(1)^{-1} = \sqrt t (1 + \gamma^{-1})$, and that $R_\infty = \eta_u(1) = 1 + \gamma^{-1}$. In the second, that $\sum_{t=1}^T \frac{1}{\sqrt t} \leq 2 \sqrt T - 1$. In the third, that $\bone < 1$. In the fourth, we used the bound on $M \leq 3 + 2 \gamma^{-1}$ from Proposition \ref{prop:bound}.
\end{proof}

It is easy to check that the previous results also hold for AMSBound (Algorithm 3 in \cite{adabound}), since no assumptions were made on the point-wise behavior of $\currv$.

\begin{rmk}
    \thmref{the:adabound-corrected} and Corollary \ref{cor:AdaBound} also hold for AMSBound.
\end{rmk}
\section{Experiments on AdaBound and SGD}
\label{sec-experiments}

Unfortunately, the regret bound in Corollary \ref{cor:AdaBound} is minimized in the limit $\gamma \to \infty$, where AdaBound immediately degenerates to SGD. To inspect whether this fact has empirical value or is just an artifact of the presented analysis, we evaluate the performance of AdaBound when training neural networks on the CIFAR dataset \citep{cifar} with an extremely small value for $\gamma$ parameter.

Note that $\gamma = 0.001$ was used for the CIFAR results in \cite{adabound}, for which we have $\eta_l(t) > 0.5$ and $\eta_u(t) < 2$ after only 3 epochs ($391$ iterations per epoch for a batch size of $128$), hence we believe results with considerably smaller/larger values for $\gamma$ are required to understand its impact on the performance of AdaBound.

We trained a Wide ResNet-28-2 \citep{wide} using the same settings in \cite{adabound} and its released code \footnote{\url{https://github.com/Luolc/AdaBound}, version 2e928c3}: $\bone = 0.9, \btwo = 0.999$, a weight decay of $0.0005$, a learning rate decay of factor 10 at epoch 150, and batch size of $128$. For AdaBound, we used the author's implementation with $\alpha = 0.001, \alpha^\star = 0.1$, and for SGD we used $\alpha = 0.1$. Experiments were done in PyTorch.

To clarify our network choice, note that the model used in \cite{adabound} is \textit{not} a ResNet-34 from \cite{resnet1}, but a variant used in \cite{cutout}, often referred as ResNet-34. In particular, the ResNet-34 from \cite{resnet1} consists of 3 stages and less than 0.5M parameters, while the network used in \cite{adabound} has 4 stages and around 21M parameters. The network we used has roughly 1.5M parameters.

Our preliminary results suggest that the final test performance of AdaBound is monotonically increasing with $\gamma$ -- more interestingly, there is no significant difference throughout training between $\gamma = 0.001$ and $\gamma = 100$ (for the latter, we have $\eta_l(1) \geq 0.99$ and $\eta_u(1) \leq 1.01$).

To see why AdaBound with $\gamma = 100$ behaves so differently than SGDM, check that the momentum updates slightly differ between the two: for AdaBound, we have:
\begin{equation}
    \currm = \bone \prevm + (1-\bone) \currg
\end{equation}
while, for the implementation of SGDM used in \cite{adabound}, we have:
\begin{equation}
\label{eq:dampening}
    \currm = \bone \prevm + (1 - \kappa) \currg
\end{equation}
where $\kappa \in [0,1)$ is the dampening factor. The results in \cite{adabound} use $\kappa = 0$, which can cause $\currm$ to be larger by a factor of $\frac{1}{1-\bone} = 10$ compared to AdaBound. In principle, setting $\kappa = \bone$ in SGDM should yield dynamics similar to AdaBound's as long as $\gamma$ is not extremely small.

Figure \ref{fig-curves} presents our main empirical results: setting $\gamma = 10^{-10}$ causes noticeable performance degradation compared to $\gamma = 0.001$ in AdaBound, as Corollary \ref{cor:AdaBound} might suggest. Moreover, setting $\kappa = \bone$ in SGDM causes a dramatic performance increase throughout training. In particular, it slightly outperforms AdaBound in terms of final test accuracy ($94.01\%$ against $93.90\%$, average over 5 runs), while being comparably fast and consistent in terms of progress during optimization.

We believe SGDM with $\kappa = \bone$ (which is currently \textit{not} the default in either PyTorch or Tensorflow) might be a reasonable alternative to adaptive gradient methods in some settings, as it also requires less computational resources: AdaBound, Adam and SGDM's updates cost $11d$, $9d$ and $5d$ float operations, respectively, and their memory costs are $3d$, $3d$ and $2d$. Moreover, AdaBound has 5 hyperparameters ($\alpha, \alpha^*, \bone, \btwo, \gamma$), while SGDM with $\kappa = \bone$ has only 2 ($\alpha, \bone$). Studying the effectiveness of `dampened' SGDM, however, requires extensive experiments which are out of the scope of this technical report.

Lastly, we evaluated whether performing bias correction on the $\kappa = \bone$ form of SGDM affects its performance. More specifically, we divide the learning rate at step $t$ by a factor of $1 - \bone^t$. We observed that bias correction has no significant effect on the average performance, but yields smaller variance: the standard deviation of the final test accuracy over 5 runs decreased from $0.21$ to $0.16$ when adding bias correction to $\kappa = \bone$ SGDM (compared to $0.27$ of AdaBound).

\begin{figure}
    \centering
    \begin{subfigure}{.5\textwidth}
      \centering
      \includegraphics[width=\linewidth]{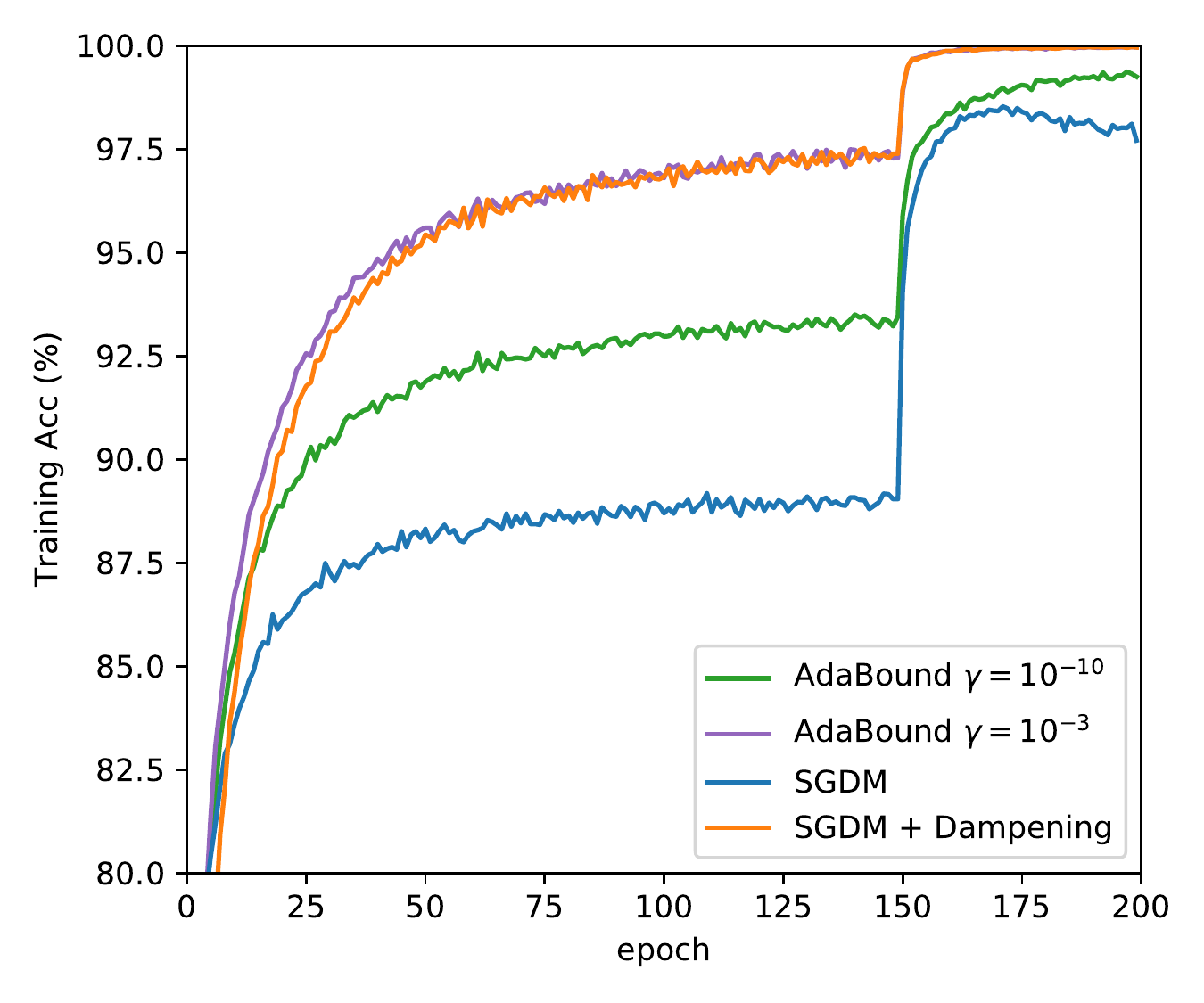}
    \end{subfigure}%
    \hfill
    \begin{subfigure}{.5\textwidth}
      \centering
      \includegraphics[width=\linewidth]{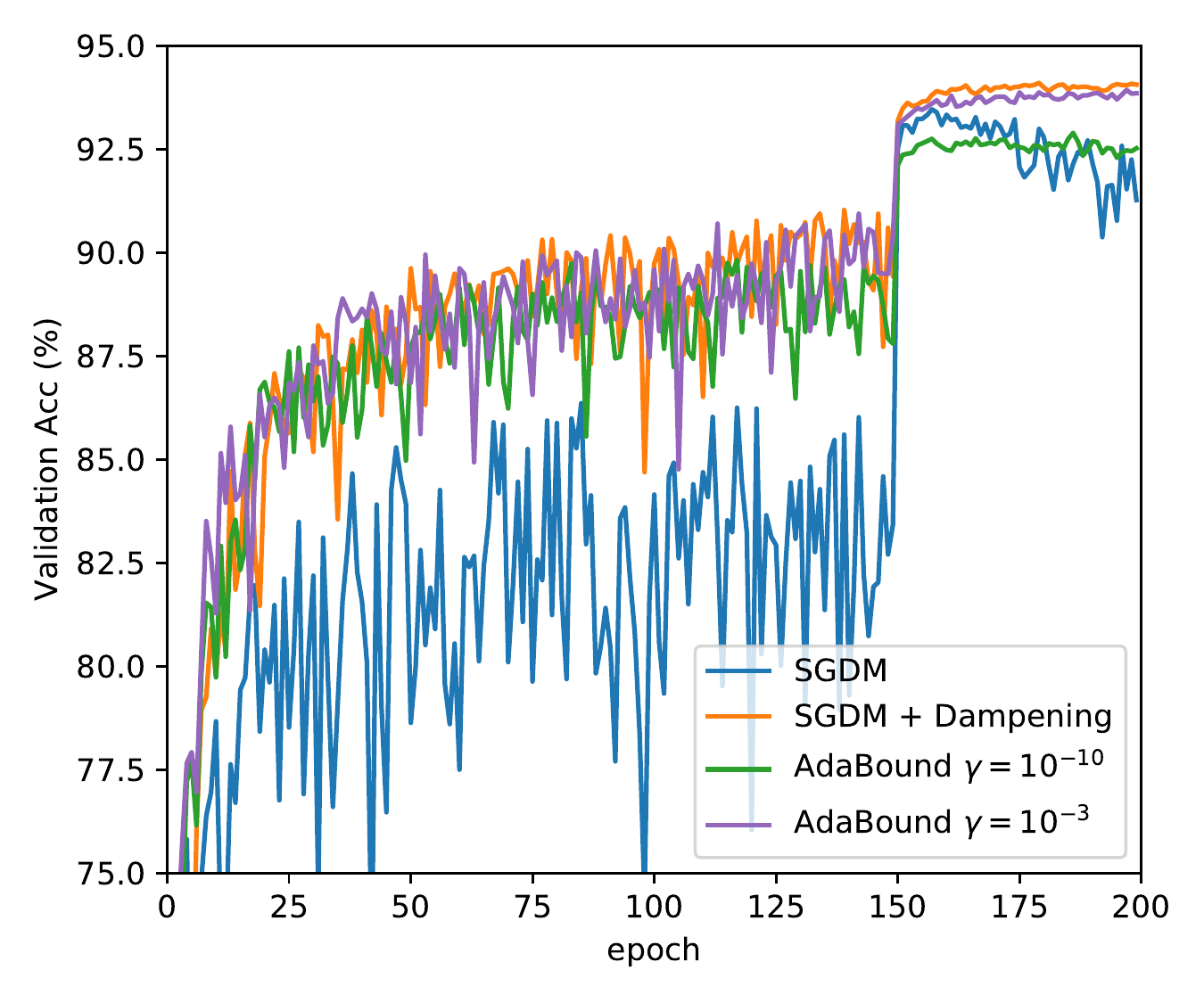}
    \end{subfigure}
    \caption{Training curves for AdaBound and momentum SGD (SGDM) (\textbf{left: } training accuracy, \textbf{right: } test accuracy) when optimizing a Wide ResNet 28-2 on the CIFAR-10 dataset. AdaBound suffers a significant performance drop both in training accuracy and final test performance with very small values of $\gamma$. Adding dampening ($\kappa = \bone$ instead of $\kappa = 0$ in \eqref{eq:dampening}) to SGDM yields a momentum update which matches AdaBound and other Adam-like methods, causing a dramatic performance increase (compare blue and orange curves) and providing performance similar to AdaBound (orange and purple curves).}
    \label{fig-curves}
    \end{figure}
\section{Discussion}

In this technical report, we identified issues in the proof of the main Theorem of \cite{adabound}, which presents a regret rate guarantee for AdaBound. We presented an instance where the statement does not hold, and provided a $O( \sqrt T)$ regret guarantee under different -- and arguably less restrictive -- assumptions.  Finally, we observed empirically that AdaBound with a theoretically optimal $\gamma$ indeed yields superior performance, although it degenerates to a specific form of momentum SGD. Our experiments suggest that this form of SGDM (with a dampening factor equal to its momentum) performs competitively to AdaBound on CIFAR.

\subsection*{Acknowledgements}

We are in debt to Rachit Nimavat for proofreading the manuscript and the extensive discussion, and thank Sudarshan Babu and Liangchen Luo for helpful comments.

\bibliography{collection}
\bibliographystyle{iclr2019_conference}

\end{document}